\documentclass[10pt]{article}

\usepackage[margin=1in]{geometry}
\usepackage{natbib}
\usepackage{parskip}

\usepackage{algorithm, algorithmicx, algpseudocode}
\usepackage{amssymb}
\usepackage{amsmath}
\usepackage{graphicx}
\usepackage{mathtools}
\usepackage{needspace} 
\usepackage{multicol}
\usepackage{listings}
\usepackage{amsthm}
\usepackage[shortlabels]{enumitem}
\usepackage[italicdiff]{physics}
\usepackage{cancel}
\usepackage[dvipsnames]{xcolor}
\usepackage{verbatim}
\usepackage{comment}
\usepackage{array}
\usepackage{tikz-cd}
\usepackage{import}
\usepackage{booktabs}
\usepackage{comment}
\usepackage{array}
\usepackage[colorlinks, linkcolor=blue, citecolor=ForestGreen]{hyperref}
\usepackage{forest}
\usepackage{float}
\usepackage{xspace}
\usepackage{aliascnt}
\usepackage{wrapfig}
\usepackage{thm-restate}
\usepackage{adjustbox}
\usepackage{bm}
\usepackage{bbm}
\usepackage{complexity}
\usepackage{makecell}
\usepackage{nicefrac}
\usetikzlibrary{calc, positioning, external, arrows.meta}
\usepackage{cleveref}

\newcommand{\ind}{q}

\allowdisplaybreaks

\let\cite\citep

\newfloat{protocol}{tbp}{lop}
\floatname{protocol}{Protocol}

\newtheorem{theorem}{Theorem}
\numberwithin{theorem}{section}
\newtheorem*{theorem*}{Theorem}

\newtheorem{observation}[theorem]{Observation}
\newtheorem{lemma}[theorem]{Lemma}

\theoremstyle{definition}

\usepackage{caption}
\makeatletter
\renewcommand\thanks[1]{%
  \footnotemark%
  \protected@xdef\@thanks{\@thanks
    \protect\footnotetext[\the\c@footnote]{#1}}%
}
\makeatother

\title{Toward Optimal Regret in Robust Pricing:\\ Decoupling Corruption and Time}
\author{
	Kalana Kalupahana\thanks{Equal contribution.}\\[-0.2em]
	\footnotesize\texttt{kalanakalpitha.kalupahana@mail.polimi.it} \\[-0.2em]
	\footnotesize Politecnico di Milano 
	\and
	Francesco Emanuele Stradi\footnotemark[1]\\[-0.2em]
	\footnotesize\texttt{francescoemanuele.stradi@polimi.it} \\[-0.2em]
	\footnotesize Politecnico di Milano 
	\and
	Matteo Castiglioni \\[-0.2em]
	\footnotesize\texttt{matteo.castiglioni@polimi.it} \\[-0.2em]
	\footnotesize Politecnico di Milano 
	\and
	Alberto Marchesi \\[-0.2em]
	\footnotesize\texttt{alberto.marchesi@polimi.it} \\[-0.2em]
	\footnotesize Politecnico di Milano
}
\date{\today}

\begin{document}

\maketitle

\begin{abstract}
	We design the first regret guarantees for robust dynamic pricing that decouple the dependence on the corruption $C$ and the time horizon $T$.
	In dynamic pricing, a seller with unlimited supply of a good interacts with a stream of buyers over \( T \) rounds, with the goal of maximizing revenue. At each round $t$, the seller posts a price $p_t$, and the buyer purchases the good only if their unknown valuation $v^\star$ exceeds this price. The seller observes only the binary feedback $\mathbbm{1} \left\{ p_t \leq v^\star \right\}$, indicating whether a sale occurred. In the \emph{robust} pricing setting, a malicious adversary is allowed to corrupt this feedback in at most $C$ rounds. 
	Even if the learner knows the corruption $C$, the best known regret bound is $\mathcal{O}(C\log\log T)$ by \citet{gupta2025robust}.
	This leaves as an open problem to ``decouple'' the dependence on $C$ and $T$.
	In this work, we resolve this open problem. In particular, we develop a robust variant of binary search that achieves regret $\mathcal{O}(C+\log T)$ when the corruption $C$ is known and $\mathcal{O}(C+\log^2 T)$ when the corruption is unknown. 
\end{abstract}

\newpage

\tableofcontents

\newpage

\section{Introduction}

\emph{Dynamic pricing} is a central problem in sequential decision-making under uncertainty, first studied by~\citet{kleinberg2003value}. In its simplest form, a seller (the learner) with unlimited supply of a good interacts with a stream of buyers over \( T \) rounds. At each round \( t \), the seller posts a price \( p_t \in [0,1] \), and the buyer purchases the good only if their \emph{unknown} valuation \( v^\star \in [0,1] \) exceeds this price, \emph{i.e.}, when \( p_t \leq v^\star \). Thus, the seller's revenue at round \( t \) is given by:
\[
r_t = p_t \cdot \mathbbm{1} \{ p_t \leq v^\star \}.
\]
The goal of the seller is to maximize the total revenue over \( T \) rounds, or equivalently, to minimize the \emph{regret} $R_T$ with respect to always posting the optimal fixed price in hindsight, namely \( v^\star \). Formally:
\[
R_T = T \cdot v^\star - \textstyle\sum_{t=1}^T r_t.
\]

One of the main challenges in dynamic pricing is that, at the end of each round \( t \), the seller observes only \emph{binary feedback}, indicating whether a purchase occurred, \emph{i.e.}, the value of \( \mathbbm{1} \{ p_t \leq v^\star \} \). 
In their seminal work,~\citet{kleinberg2003value} show that such problems can be addressed very efficiently by designing a clever binary-search-style algorithm that achieves an optimal regret rate of \( \Theta(\log\log T) \). Since then, the dynamic pricing problem and several of its variants have been extensively studied (see Appendix~\ref{sec:related_works} for additional discussion on related work).

In this paper, we study the \emph{robust} version of the dynamic pricing problem described above---originally introduced by~\citet{krishnamurthy2023contextual}---in which the feedback received by the seller is subject to \emph{adversarial corruptions}. In particular, in at most \( C \) rounds, the feedback observed by the seller may be arbitrary. The goal is thus to derive regret guarantees that degrade gracefully as a function of \( C \).

Adversarial corruptions arise naturally in many real-world scenarios, where the observed sale/no-sale signal may fail to reflect the true comparison between the posted price and the buyer's valuation due to noise, strategic behavior, or manipulation. Remarkably, dealing with corruption is especially delicate in pricing. Since each round reveals only one bit of information, even a small number of corrupted observations may mislead the learning process. Put differently, the same structure that makes the uncorrupted problem highly learnable also renders it intrinsically fragile: once some of the comparisons are unreliable, naively extending binary search ideas is no longer sufficient.

The state-of-the-art approach for this setting is the recent algorithm of \citet{gupta2025robust}, which achieves regret $\mathcal{O}(C\log\log T)$---their result also extends to the $d$-dimensional contextual setting; see Appendix~\ref{sec:related_works} for a more detailed discussion.
Interestingly, they also show that regret guarantees of order \( \mathcal{O}(C + \log\log T) \) are unattainable, implying that obtaining a linear dependence on $C$ requires a departure from the classical $\mathcal{O}(\log\log T)$ rate typical of uncorrupted dynamic pricing.
Moreover, they achieve this linear dependence on $C$ in a restricted setting with \emph{cautious buyers}, in which corruptions are \emph{one-sided}: a buyer may refuse to purchase even when the posted price is below their valuation, but can never purchase when the price exceeds it. In this model and assuming to know the corruption $C$, they obtain a regret guarantee of order \( \mathcal{O}(C + \log T) \).
They leave open whether the same regret bound can be achieved under arbitrary but known corruption. We go beyond that and answer in the affirmative the following question:
\begin{center}
	\emph{Can regret $\mathcal{O}(C)+\textnormal{poly}(\log T)$ be achieved in dynamic pricing under arbitrary and unknown adversarial corruption?}
\end{center}

\subsection{Our Results and Techniques}
Motivated by the problem left open by~\citet{gupta2025robust}, we focus on regret guarantees of order
$\mathcal{O}(C)+ \text{poly}(\log T)$. This rate is provably optimal in the large-corruption regime
$C\geq \text{poly}(\log T)$, as suggested by the following simple observation.
\begin{observation}\label{obs}
	Consider the robust dynamic pricing problem with feedback corrupted in at most $C$ rounds.
	Then, any deterministic algorithm suffers regret $R_T\geq \Omega(C)$,
	even if $C$ is known.
\end{observation}
The intuition is the following. Consider two instances with well-separated valuations, say
$v^\star_L=1/3$ and $v^\star_H=2/3$. For the first $C$ rounds, the adversary can use its corruption budget
to make the feedback observed under the second instance identical to the uncorrupted feedback that would have been
observed under the first one. Since the algorithm is deterministic,
the same prices are posted in both instances during these rounds.
However, for any posted price $p_t$, at least one of the two instances incurs constant regret. Hence, the sum of the regrets over the two instances is $\Omega(C)$, which
implies that one of them has regret $\Omega(C)$.

In this work, we provide regret guarantees for both known and unknown corruption $C$. 
At the core of both results lies a unified meta-algorithm, which is inspired by a classical binary search and then robustified through two additional components: \emph{safety checks} and \emph{backtracking}. This approach is primarily motivated by the impossibility result of \citet{gupta2025robust}, which shows that the $\mathcal O(\log\log T)$ rate achieved by clever approaches in the uncorrupted setting (see, \emph{e.g.},~\citep{kleinberg2003value}) is unattainable in the presence of adversarial corruption. This suggests that a classical binary-search-style approach---yielding a $\log T$\footnote{In this work, we will refer as $\log T$ to $\log_2 T.$} dependence---is the appropriate algorithmic approach for pricing under corruption.

The algorithm begins with a binary search phase designed to shrink the candidate interval.
This phase is naturally agnostic to the corruption level $C$ and is shared across both settings. We employ a potential function argument to prove that the number of ``wrong steps''---in which the algorithm narrows the search to an interval not containing $v^\star$---is bounded linearly in the corruption $C$.

Once the algorithm reaches a leaf interval of size $1/T$, it enters a ``commitment phase'' in which it tries to determine whether it is committing to the right interval. If it is the case, the algorithm can simply post the lower extreme, incurring a negligible per-round regret of at most $1/T$.
This step is tailored to the known or unknown corruption setting.

In the known corruption setting, the algorithm keeps posting the left and right extremes of the interval. If the feedback remains consistent with $v^\star$ belonging to the interval for $C+1$ times, we can safely assume that $v^\star$ belongs to the interval and thus starts committing to the left extreme. Any inconsistent feedback triggers a backtracking procedure, returning the search to the parent node, and restarting the binary search procedure from the parent node. We upperbound the number of rounds required to commit to the optimal interval and show that this strategy yields a regret bound of $\mathcal{O}(C + \log T)$, resolving the open problem posed by \citet{gupta2025robust}.

In the unknown-corruption setting, designing the commitment phase is more challenging. 
The algorithm must balance a fundamental tension: to quickly leave an incorrect leaf, it should query the right endpoint, but if the leaf is correct, these queries only lead to unnecessary regret. 
To address this, we design a randomized commitment strategy in which the algorithm queries the right endpoint with a probability that decreases over time.
The intuition is that, as a leaf survives more safety checks, either it is increasingly plausible that the leaf is correct, or a large number of corruptions have been used to make it appear correct. 
In both cases, the algorithm can reduce exploration: in the first case, because committing is likely safe, and in the second, because the extra regret can be charged to the corruption budget. 
By carefully scheduling the decaying probability, we balance the regret from staying in a wrong leaf against the regret from exploring in a correct one, obtaining a total regret bound of $\mathcal{O}(C+\log^2 T)$.

\subsection{Relation to Robust Binary Search}

Robust pricing is deeply connected with binary search problems under corruption~\citep{rivest1980coping,dagan2021entropy}.
Specifically, in standard, uncorrupted dynamic pricing, binary search is not optimal: the seminal result of~\citet{kleinberg2003value} shows that the threshold structure of the problem can be exploited much more efficiently, leading to the optimal \(\Theta(\log\log T)\) regret rate. 
However, as discussed above, this doubly logarithmic dependence is no longer compatible with a linear dependence on the corruption budget \(C\) under adversarial corruptions. 
Once the goal becomes to obtain guarantees of the form
$
\mathcal O(C)+\operatorname{poly}(\log T),
$
the problem naturally reduces to identify an approximately optimal solution as quickly as possible,  upper-bounding the total regret by the length of the learning phase.
Indeed, suppose that the learner can identify an interval of length at most \(1/T\) containing \(v^\star\). 
Then, by repeatedly posting the left endpoint of this interval, the learner guarantees a sale and suffers at most \(1/T\)  per-round regret. 
Identifying such small interval is the objective of robust binary search: each price query can be viewed as a comparison with \(v^\star\), and a corrupted observation corresponds to a comparison answer pointing to the wrong side.

This perspective connects our setting to the seminal work of \citet{rivest1980coping}, who study binary search with erroneous answers. 
In their model, at most \(C\) comparison answers may be corrupted, and the value of \(C\) is known to the algorithm. 
Their main result gives a worst-case query complexity of
$
\log T + C\log\log T + \mathcal O(C\log C).
$
Although this is the form in which their result is presented, a more refined analysis of their actual upper bound (which we provide in Appendix~\ref{app:robust_binary_search}) can lead to a worst-case query complexity of 
$
\mathcal{O}(C + \log T ).
$
Combined with the commit-to-the-left-endpoint argument above, this yields an alternative proof of the \(\mathcal O(C+\log T)\) regret guarantee for robust pricing in the known-corruption case.

The limitation of this approach is that it crucially relies on knowing the corruption budget in advance, and it cannot be easily extended to the case $C$ is \emph{not} known.
This motivates the approach developed in the present paper. 
We design a robust search procedure that preserves the same high-level goal---finding a \(1/T\)-accurate interval and then exploiting its left endpoint---but is built such that it can be extended to the unknown-corruption setting. 

\section{Preliminaries}
\label{sec:preliminaries}

We consider a \emph{dynamic pricing problem} over a horizon of \( T \) rounds, denoted by \( [T] = \{1, \ldots, T\} \), during which a seller with unlimited supply of a good interacts with a stream of buyers who share a common but unknown valuation \( v^\star \in [0,1) \) for one unit of the good.\footnote{We assume that \( v^\star \in [0,1) \) to avoid boundary cases in our analysis. All our results can be readily extended to the case \( v^\star \in [0,1] \) using standard techniques.} At each round \( t \in [T] \), the seller posts a price \( p_t \in [0,1] \) based on past observations. The buyer then decides whether to purchase the item according to a standard threshold rule: a sale occurs if and only if the posted price does not exceed the valuation. Accordingly, the seller observes binary feedback:
\[
\sigma_t := \mathbbm{1}\{p_t \le v^\star\},
\]
where $\sigma_t = 1$ indicates a successful sale and $\sigma_t = 0$ indicates no sale. The revenue collected by the seller at each round \( t \in [T] \) is therefore \( r_t := p_t \cdot \sigma_t \), and the seller’s goal is to maximize the cumulative revenue over the entire time horizon.

The seller's performance is evaluated in terms of their \emph{(cumulative) regret}, which measures how much revenue the seller loses compared to a clairvoyant benchmark that knows the true valuation \( v^\star \) and posts it at every round. Since posting \( v^\star \) always results in a sale and yields revenue exactly \( v^\star \) per round, the cumulative benchmark revenue is \( T \cdot v^\star \). Thus, the seller’s regret is defined as:
\[
R_T := T \cdot v^\star - \sum_{t=1}^T r_t.
\]

\subsection{Robust Dynamic Pricing}

In this paper, we study a \emph{robust} extension of the dynamic pricing model introduced above, in which the feedback observed by the seller may be corrupted by an adversary. Following the standard formulation for adversarial corruptions in pricing~\citep{krishnamurthy2023contextual, gupta2025robust}, we assume that the total number of corrupted rounds is bounded by a budget \( C \). We consider two settings of increasing difficulty: the \emph{known-corruption setting}, in which the seller knows the budget \( C \), and the \emph{unknown-corruption setting}, in which \( C \) is not known to the seller.

To distinguish between the truthful feedback generated according to \( v^\star \) and the feedback actually observed by the seller, we let:
\[
y_t := \mathbbm{1}\{ p_t \le v^\star \}
\]
denote the uncorrupted sale indicator, while \( \sigma_t \in \{0,1\} \) denotes the feedback actually observed by the seller.  
In the robust setting, \( \sigma_t \) may differ from \( y_t \) subject to the following budget constraint:
\[
\left| \{ t \in [T] : \sigma_t \neq y_t \} \right| \le C.
\]
On corrupted rounds, a truthful sale ($y_t=1$) may be reported as a no-sale ($\sigma_t=0$), and a truthful no-sale ($y_t=0$) may be reported as a sale ($\sigma_t=1$). On all uncorrupted rounds, the feedback is consistent with the threshold rule induced by $v^\star$.
As is standard in the literature~\citep{krishnamurthy2023contextual}, we assume that the adversary can decide whether to corrupt a round \( t \) based on the history and the seller’s distribution over prices at time \( t \), but not on the realized price.\footnote{We allow the adversary to decide \emph{how} to corrupt the feedback depending on the actual price selected by the seller.}

Our goal is to design dynamic pricing strategies for the seller that achieve regret guarantees depending on the corruption budget \( C \).
\section{Meta Algorithm}
\label{sec:algorithm}

\begin{wrapfigure}[34]{R}{0.55\textwidth}
	\vspace{-0.4cm}
	\begin{minipage}{0.55\textwidth}
		\begin{algorithm}[H] 
			\caption{Robust Dynamic Pricing Meta-Algorithm}
			\label{alg:meta}
			\begin{algorithmic}[1]
				\State $I \gets I_0=[0,1]$ \label{line:init-interval}
				\State $i \gets 0$
				\While{\textsc{TRUE}}
				\State  $I\gets I_{i}=[L,R)$
				\If{$I \notin \mathcal{L}$}
				\If{\textsc{SafetyCheck}$(I)=$ \textsc{FAIL}} \label{line:internal-fail}
				\State $i\gets i-1$
				\Else
				\State Post $M=\frac{L+R}{2}$ and observe $\sigma_M$ \label{line:midpoint-query}
				\State $i\gets i+1$
				\If{$\sigma_M=1$}
				\State $I_i \gets [M,R]$ \label{line:right-update}
				\Else
				\State $I_i \gets [L,M]$ \label{line:left-update}
				\EndIf
				\EndIf
				\Else
				\If{\textsc{Commit(I)}$=$ \textsc{FAIL}}
				\State $i \gets i-1$        
				\EndIf
				\EndIf
				\EndWhile
			\end{algorithmic}
		\end{algorithm}
		\vspace{-0.5cm}
		\begin{algorithm}[H]
			\caption{\textsc{SafetyCheck}}
			\label{alg:safety}
			\begin{algorithmic}[1]
				\Require Current interval $I=[L,R)$
				\State Post $L$ and observe $\sigma_L$
				\State Post $R$ and observe $\sigma_R$
				\If{$\sigma_L=0$ or $\sigma_R=1$}
				\State \Return \textsc{FAIL}
				\Else
				\State \Return \textsc{PASS}
				\EndIf
			\end{algorithmic}
		\end{algorithm}
		
	\end{minipage}
\end{wrapfigure}

In this section, we present a meta-algorithm for dynamic pricing under adversarial corruptions. In subsequent sections, we instantiate it with different components depending on whether the corruption budget is known or unknown to the seller.

Our approach is motivated by the impossibility result of \citet{gupta2025robust}, which shows that the $\mathcal O(\log\log T)$ regret achieved by clever approaches in the uncorrupted setting (see~\citep{kleinberg2003value}) is unattainable in the presence of adversarial corruptions.
Consequently, we build our dynamic pricing strategy upon a classic binary search, ``robustified'' through two additional components: \emph{(i) safety checks} and \emph{(ii) backtracking}.

In \Cref{alg:meta}, we provide the pseudocode of our meta-algorithm.
The algorithm begins with the unit interval \( I_0 := [0,1] \) and iteratively shrinks it to a ``leaf'' interval \( I_D \) of length at most \( 1/T \), which is reached at depth \( D := \log T \).\footnote{In principle, it should be \( D := \lceil \log T \rceil \). Throughout the paper, for ease of presentation, we often omit the \( \lceil \cdot \rceil \) operator.}
The algorithm operates under the (possibly incorrect) assumption that the true valuation \( v^\star \) lies within the current interval. It proceeds by shrinking the interval according to a standard binary search, augmented with additional safety checks. If the feedback becomes inconsistent with past observations, the algorithm triggers a backtracking procedure.

Given the current interval \(I = [L,R)\), the algorithm first performs a safety check by querying the endpoints $L$ and $R$. Let $\sigma_L$ and $\sigma_R$ denote the observed feedback. If $\sigma_L = 0$ (no sale at the lower extreme) or $\sigma_R = 1$ (sale at the upper extreme), the observations are inconsistent with the hypothesis $v^\star \in [L, R]$.\footnote{To handle boundary cases in which corruption is trivially evident---specifically, when \( L = 0 \) and \( \sigma_L = 0 \), or \( R = 1 \) and \( \sigma_R = 1 \)---we assume that the checks at \( L = 0 \) and \( R = 1 \) pass by default.}
In this case, the algorithm backtracks to the previous iteration of the binary search, \emph{i.e.}, the ``parent'' interval. If the safety check passes, the algorithm queries the midpoint \( M := \frac{L + R}{2} \). Based on the feedback \( \sigma_M \) obtained by posting \( M \), it updates the interval according to the standard binary search rule.

The process described above continues until the algorithm reaches a leaf interval in $\mathcal{L} := \{\ell_1, \dots, \ell_{2^D}\}$ at depth $D$. Notice that all the leaf intervals have length at most $1/T$. We denote the unique leaf interval containing $v^\star$ as $\ell^\star$.

Once a leaf interval is reached, the interval is sufficiently small that no further shrinking is required. The algorithm therefore enters a \emph{commitment phase}. For now, we consider an arbitrary committing strategy; its precise implementation depends on the specific setting---namely, whether the corruption budget \( C \) is known or unknown. At each round, this procedure may either continue committing or return a \textsc{FAIL} signal, in which case the algorithm backtracks to the parent interval.
Crucially, aside from the commitment phase, our algorithm does not require knowledge of $C$. As we will see in the following sections, this knowledge can instead be leveraged to design simpler and more effective commitment strategies.

In Figure~\ref{fig:robust-binary-tree}, we provide a graphical representation of the ``robustified'' binary search performed by our algorithm.

\begin{figure*}[t] 
	\centering
	\resizebox{\textwidth}{!}{%
		\begin{tikzpicture}[
			x=1cm,y=1cm,
			interval/.style={
				draw, rounded corners=2pt, thick, fill=blue!6,
				minimum width=2.9cm, minimum height=0.95cm,
				align=center, font=\Large
			},
			final/.style={
				draw, rounded corners=2pt, thick, fill=orange!12,
				minimum width=2.9cm, minimum height=1.0cm,
				align=center, font=\Large
			},
			pass/.style={
				-{Latex[length=2.4mm]}, thick, teal!65!black
			},
			fail/.style={
				-{Latex[length=2.4mm]}, thick, dashed, red!75!black
			},
			helper/.style={
				-{Latex[length=2.0mm]}, thick, gray!70
			},
			lab/.style={
				font=\Large,
				fill=white,
				inner sep=1.2pt,
				text opacity=1,
				fill opacity=0.95
			},
			onpathinterval/.style={
				interval,
				draw=green!55!black,
				fill=green!12,
				line width=1.4pt
			},
			onpathfinal/.style={
				final,
				draw=green!55!black,
				fill=green!12,
				line width=1.4pt
			}
			]
			
			\node[onpathinterval] (root) at (0,0) {$I_0=[0,1)$};
			
			\node[interval] (L1) at (-6,-2.6) {$[0,\tfrac12)$};
			\node[onpathinterval] (R1) at ( 6,-2.6) {$[\tfrac12,1)$};
			
			\node[interval] (LL) at (-10,-5.2) {$[0,\tfrac14)$};
			\node[interval] (LR) at ( -2,-5.2) {$[\tfrac14,\tfrac12)$};
			\node[onpathinterval] (RL) at (  2,-5.2) {$[\tfrac12,\tfrac34)$};
			\node[interval] (RR) at ( 10,-5.2) {$[\tfrac34,1)$};
			
			\node at (-10,-6.7) {$\vdots$};
			\node at ( -2,-6.7) {$\vdots$};
			\node at (  2,-6.7) {$\vdots$};
			\node at ( 10,-6.7) {$\vdots$};
			
			\node[final] (F1) at (-10,-8.4) {$\ell_1$};
			\node[final] (F2) at ( -2,-8.4) {$\ell_{j-1}$};
			\node[onpathfinal] (F3) at (  2,-8.4) {$\ell_j=\ell^\star$};
			\node[final] (F4) at ( 10,-8.4) {$\ell_{2^D}$};
			
			\draw[pass] (root) -- (L1)
			node[midway, above left=2pt and 1pt, lab] {$\sigma_M=0$};
			
			\draw[pass] (root) -- (R1)
			node[midway, above right=2pt and 1pt, lab] {$\sigma_M=1$};
			
			\draw[pass] (L1) -- (LL);
			\draw[pass] (L1) -- (LR);
			\draw[pass] (R1) -- (RL);
			\draw[pass] (R1) -- (RR);
			
			\draw[helper] (LL) -- ++(0,-0.9);
			\draw[helper] (LR) -- ++(0,-0.9);
			\draw[helper] (RL) -- ++(0,-0.9);
			\draw[helper] (RR) -- ++(0,-0.9);
			
			\draw[helper] (-10,-7.1) -- (F1);
			\draw[helper] ( -2,-7.1) -- (F2);
			\draw[helper] (  2,-7.1) -- (F3);
			\draw[helper] ( 10,-7.1) -- (F4);
			
			\draw[very thick, green!55!black] (root) -- (R1);
			\draw[very thick, green!55!black] (R1) -- (RL);
			\draw[very thick, green!55!black] (  2,-7.1) -- (F3);
			
			
			\draw[fail, bend left=28] (LR.north west)
			to node[pos=0.46, left, lab] {\textsc{FAIL} / backtrack}
			(L1.south);
			
			\node[
			draw=orange!60!black, dashed, rounded corners=4pt,
			fit=(F1)(F4), inner sep=6pt
			] {};
			
			\node[font=\Large, text=red!75!black] at ($(L1)+(-2.0,0.8)$) {off-path};
			\node[font=\Large, text=green!55!black] at ($(R1)+(2.0,0.8)$) {on-path};
			\node[font=\Large, text=red!75!black] at ($(RR)+(2.0,0.8)$) {off-path};
			
			
		\end{tikzpicture}%
	}
	\caption[Robust binary search]{Tree of intervals of our robust dynamic pricing algorithm.
		The intervals induced by repeated halving $I_0=[0,1)$ can be represented as a binary tree of depth $D=\lceil \log_2 T\rceil$. The green nodes are the on-path intervals leading to the correct leaf $\ell^\star$, while the remaining intervals are off-path. The intervals at depth $D$ are the leaves $\ell_j$ of the tree.}

\label{fig:robust-binary-tree}
\end{figure*}

\subsection{Regret Decomposition and Potential Argument}

In this section, we provide a regret decomposition for \Cref{alg:meta}. Then, by combining this decomposition with the analysis of the commitment phases in \Cref{sec:known} for known $C$, and in \Cref{sec:unknown} for unknown $C$, we obtain our final regret bounds.

To analyze the cumulative regret, we partition the horizon \( T \) into \emph{search steps}. Each search step corresponds either to \emph{(i)} a single iteration of the binary search, consisting of a safety check and a possible midpoint query, or to \emph{(ii)} a sequence of consecutive rounds in the commitment phase.

We introduce some additional notation. For any $\ell \in \mathcal{L}$, let $Q(\ell)$ denote the set of rounds in which the algorithm runs the commitment subroutine on $\ell$. The regret incurred over these rounds is denoted by:
\[
R(\ell):= \sum_{t\in Q(\ell)} \left(v^\star- p_t \mathbbm{1}\{p_t\le v^\star\}\right).
\]
Finally, let $N_{\textsc{T}}$ and $N_{\textsc{F}}$ be the number of times the algorithm fails the commitment check on the correct leaf and on an incorrect leaf, respectively. In this section, we aim to show that the total regret can be bounded as:
\[
R_T\le \mathcal{O}(C + \log T+ N_{\textsc{T}})+ \sum_{\ell \in \mathcal{L}} R(\ell).
\]
We employ a potential argument to track the progress of the search. For any interval $I$ in the binary tree, let $\mathrm{dist}(I, \ell^\star)$ denote the length of the shortest path between $I$ and $\ell^\star$. At search step $\ind$, let $I^\ind$ be the current interval, and define the potential:
\[\Phi_\ind:= \mathrm{dist}(I^\ind, \ell^\star).\]

Initially, the algorithm starts at the root $I_0$, whose distance from any leaf is exactly $D$; hence, $\Phi_0 = D$. We say that an interval is \emph{on-path} if it lies on the shortest path from the root to $\ell^\star$, and \emph{off-path} otherwise. Moreover, we call a search step \emph{honest} if all feedback observed during the step is uncorrupted, and \emph{corrupted} if at least one round in the step is corrupted.

Now, we analyze how different search steps affect the potential $\Phi_\ind$.
In the following two lemmas, we start by considering non-leaf intervals.

\begin{restatable}{lemma}{honeststep}\label{lem:honest-step}
Every honest search step $\ind$ at a non-leaf interval $I^\ind \notin \mathcal{L}$ decreases the potential by exactly one, i.e., $\Phi_{\ind+1}= \Phi_\ind-1$. 
\end{restatable}

\begin{restatable}{lemma}{corruptedstep}
\label{lem:corrupted-step}
Every corrupted search step $\ind$ at a non-leaf interval $I^\ind \notin \mathcal{L}$ increases the potential by at most one.
\end{restatable}
Finally, we consider how commitment search steps affect the potential. 
\begin{restatable}{lemma}{leaf}\label{lem:leaf}
Every commit search step $\ind$ at a leaf interval $\ell^\ind\in \mathcal{L}$:
\begin{itemize}
	\item Reduces the potential by $1$ if $\ell^\ind \neq \ell^\star$
	\item Increases the potential by $1$ if $\ell^\ind = \ell^\star$.
\end{itemize}
\end{restatable}
Given the previous results, we are now able to provide the promised regret decomposition for our meta-algorithm. This is done in the following theorem.
\begin{restatable}{theorem}{meta}\label{thm:meta}
Assume that the feedback is corrupted in at most $C$ rounds. Then, Algorithm~\ref{alg:meta} limits the number of commitment steps on incorrect leaves to $N_{\textsc{F}}\le \log T+ C+ N_{\textsc{T}}$. Furthermore, its regret is at most: 
\[
R_T\leq \sum_{\ell \in \mathcal{L}} R(\ell) + 3\log T+6 C +3 N_{\textsc{T}}.
\]
\end{restatable}
Intuitively, \Cref{thm:meta} shows that the regret of the meta-algorithm is essentially the regret of the commitment subroutines, plus a small overhead due to searching. The key point is that failed commitments on incorrect leaves cannot occur too often: each such failure decreases the distance to the correct leaf, while this distance can increase only because of corrupted search steps or failed commitments on the correct leaf. Since the initial distance is only $D=\log T$, this yields an overhead of order $\log T+C+N_{\textsc{T}}$. Hence, after this decomposition, it remains to control the leaf-level regret terms $R(\ell)$ and the number $N_{\textsc{T}}$ of failed commitment checks on the correct leaf.

\section{Warm-Up: The Known Corruption Case} \label{sec:known}

In this section, we design a commitment strategy for the setting where the learner knows the corruption budget $C$. The procedure \textsc{CommitKnown} (see \Cref{alg:commitK} for the pseudocode) employs a leaf-specific counter $s_\ell$ for each leaf $\ell$, which tracks how many times the interval has undergone a safety check at its extremes. We can think of the variables $s_\ell$ as global counters: they are initialized to $0$ at the beginning of the horizon and are shared across all calls to the commitment procedure.

\begin{wrapfigure}[17]{R}{0.55\textwidth}
\vspace{-0.8cm}
\begin{minipage}
	{0.55\textwidth}
	\begin{algorithm}[H]
		\caption{\textsc{CommitKnown}}
		\label{alg:commitK}
		\begin{algorithmic}[1]
			\Require Corruption $C$, current leaf interval $\ell=[L,R)\in \mathcal{L}$, counters $s_\ell$ for each $\ell \in \mathcal{L}$ 
			\While{\textsc{TRUE}}
			\If {$s_\ell \le C $}
			\State Post $L$ and observe $\sigma_L$
			\State Post $R$ and observe $\sigma_R$
			\If{$\sigma_L=0$ or $\sigma_R=1$}
			\State \Return \textsc{FAIL}
			\Else \State $s_{\ell}\gets s_\ell+1$
			\EndIf
			\Else
			\State Post L
			\EndIf
			\EndWhile
		\end{algorithmic}
	\end{algorithm}
\end{minipage}
\end{wrapfigure}

If the counter $s_\ell$ has not yet reached the threshold $C+1$, the algorithm performs endpoint queries as in the \textsc{SafetyCheck} procedure. If this check returns \textsc{FAIL}, the commitment subroutine terminates and signals the meta-algorithm to backtrack to the parent node. Otherwise, the counter $s_\ell$ is incremented, and the algorithm repeats the safety check.

Once the counter $s_\ell$ reaches the threshold $C+1$, the algorithm enters the ``true commitment'' phase, posting the price $L$ in every round. The intuition is straightforward: for any incorrect leaf $\ell$, an honest safety check would return \textsc{FAIL}, and therefore the adversary can prevent this from happening only by corrupting the endpoint feedback. Since this can occur at most $C$ times, a leaf that passes $C+1$ safety checks must be the correct leaf $\ell^\star$. At this point, the algorithm can safely stop performing safety checks and repeatedly post $L$, which guarantees per-round regret at most $1/T$ since $v^\star-1/T<L<v^\star$.

To prove the regret upper bound, we start from the regret decomposition result in \Cref{thm:meta}, which partitions the regret into three components: \emph{(i)} an $\mathcal{O}(\log T + C)$ regret term that is independent of the specific commitment procedure, \emph{(ii)} the number of backtracks on the correct leaf $N_{\textsc{T}}$, and \emph{(iii)} the regret on leaf intervals $\sum_{\ell \in \mathcal{L}} R(\ell)$.

Bounding $N_\textsc{T}$ is straightforward. Indeed, the algorithm backtracks from a commit step on $N_{\textsc{T}}$ only if the check on one of the two extremes is corrupted. Hence, the following observation holds.
\begin{observation}\label{obs:comK}
Algorithm \ref{alg:meta} with \textnormal{\textsc{Commit}} procedure \textnormal{\textsc{CommitKnown}} guarantees:
\[ N_{\textsc{T}}\le C.  \]
\end{observation}
Next, we bound the regret incurred while the algorithm is in the commitment phase, regardless of whether it is on the correct leaf or on an incorrect one. Recall that \Cref{thm:meta} guarantees that $N_{\textsc{F}}$ is small. The result is provided in the following lemma.
\begin{restatable}{lemma}{RCK}\label{lm:RCK}
Algorithm \ref{alg:meta} with \textnormal{\textsc{Commit}} procedure \textnormal{\textsc{CommitKnown}} guarantees:
\[\sum_{\ell \in \mathcal{L}} R(\ell)\le  2N_{\textsc{F}}+6C+3.\]
\end{restatable}
Intuitively, the proof separates the commitment regret on incorrect leaves from the one on the correct leaf. On an incorrect leaf, every successful safety check must be supported by at least one corruption; hence, apart from the final check that may trigger a backtrack, the number of commitment iterations on incorrect leaves is charged to the corruption budget. On the correct leaf, safety checks can fail only because of corrupted feedback, so after at most $C$ such failures the algorithm reaches the true commitment phase and repeatedly posts a price with regret at most $1/T$ per round.

We conclude the section by providing the final regret guarantee when $C$ is known to the learner. This is done in the following theorem, which is obtained by combining the previous results.
\begin{restatable}{theorem}{known}
Assume that the feedback is corrupted in at most $C$ rounds and $C$ is known to the learner. Algorithm~\ref{alg:meta} with \textnormal{\textsc{Commit}} procedure \textnormal{\textsc{CommitKnown}} guarantees:
\[R_T\le  \mathcal{O}\left(C+\log T\right).\]
\end{restatable}

\section{Handling Unknown Corruption}\label{sec:unknown}

The setting with unknown corruption is arguably more challenging. To the best of our knowledge, for all the algorithms for this setting, the regret includes the product between the corruption $C$ and a $T$-dependent component. In particular, the best result is the $\mathcal{O}\left((C+\log T) \log \log T\right)$ upper bound of~\citet{gupta2025robust}.

In this setting, we need to balance the tension between two components. On the one hand, we want to exit quickly from a commit step on a wrong leaf,---in particular when the prices associated to that leaf are smaller than $v^\star$--- which requires to play the right extreme. On the other hand, if we are in the optimal leaf, we want to play the right extreme as little as possible to avoid a large regret. To do that, we design a randomized commit strategy which works as follows.

\begin{wrapfigure}[23]{R}{0.55\textwidth}
\vspace{-0.9cm}
\begin{minipage}
	{0.55\textwidth}
	\begin{algorithm}[H]
		\caption{\textsc{CommitUnknown}}
		\label{alg:commitunknown}
		\begin{algorithmic}[1]
			\Require Current leaf interval $\ell=[L,R)\in \mathcal{L}$, counters $s_\ell$ for each $\ell \in \mathcal{L}$, $\delta\in(0,1)$
			\While{\textsc{TRUE}}
			\State Post $L$ and observe $\sigma_{L}$
			\If{$\sigma_L = 0$}
			\Return \textsc{FAIL}
			\EndIf
			\State $s_\ell\gets s_\ell+1$
			\State Sample $B\sim \text{Bern}
			\left(\min\left\{\frac{4\ln(\nicefrac{T}{\delta})}{s_\ell},1\right\}\right)$
			\If{$B = 0$}
			\State Post $L$ and observe $\sigma_{L}$
			\If{$\sigma_L = 0$}
			\Return \textsc{FAIL}
			\EndIf
			\Else
			\State Post $R$ and observe $\sigma_{R}$
			\If{$\sigma_R = 1$}
			\State $s_\ell\gets\max\{0,s_\ell-4/q_{s_\ell}\}$
			\State \Return \textsc{FAIL}
			\EndIf
			\EndIf
			\EndWhile
		\end{algorithmic}
	\end{algorithm}
\end{minipage}
\end{wrapfigure}

We provide the pseudocode of our commitment procedure for the unknown-$C$ setting in Algorithm~\ref{alg:commitunknown}. Specifically, \textsc{CommitUnknown} commits to different search blocks until an inconsistency is detected, in which case it returns \textsc{FAIL}. Each search block consists of two rounds. In the first round, the algorithm posts the left price $L$ and observes the corresponding feedback. If $\sigma_L=0$, then an inconsistency is detected and the algorithm returns \textsc{FAIL}. In the second round, the algorithm \emph{explores} the right price with probability
$q_{s_\ell}\coloneqq \min\left\{\frac{4\ln(\nicefrac{T}{\delta})}{s_\ell},1\right\}.$
Otherwise, the algorithm selects $L$ for the additional round. Again, \textsc{CommitUnknown} returns \textsc{FAIL} whenever an inconsistency is detected. If posting $R$ triggers a failure, we decrease $s_\ell$ by $4/q_{s_\ell}$, clipped at zero. 

Finally, we remark that the counters $s_\ell$ are initialized to zero and shared across different calls to Algorithm~\ref{alg:commitunknown}, without any other reinitialization.

As for the known corruption case, the commit phase exits the optimal leaf only if a round is corrupted.
\begin{observation} \label{obs:unknown}
Under commit procedure \textnormal{\textsc{CommitUnknown}}, we have:
\[N_{\textsc{T}}\le C.\]
\end{observation}
Thus, we are left to bound the regret associated with the leaves, $\sum_{\ell \in \mathcal{L}} R(\ell)$, in order to obtain the final bound on the total regret $R_T$. Intuitively, this is done by splitting $\mathcal{L}$ into three components: the set $\mathcal{L}^+$ of intervals above $\ell^\star$, the correct leaf $\ell^\star$, and the set $\mathcal{L}^-$ of intervals below $\ell^\star$, and then bounding the corresponding regret terms separately.

We start bounding the regret on the optimal leaf.
\begin{restatable}{lemma}{unknownoptimal}\label{lm:RCK_unknown_optimal}
Let $\delta\in(0,1)$. Algorithm \ref{alg:meta} with \textnormal{\textsc{Commit}} procedure \textnormal{\textsc{CommitUnknown}} guarantees:
\[ R(\ell^\star)\le  1+44C+20\ln T\,\ln(T/\delta),\]
with probability at least $1-\delta/3$.
\end{restatable}
Intuitively, the regret on the correct leaf $\ell^\star=[L,R)$ is essentially due to the rounds in which the algorithm posts the right endpoint $R$. Indeed, by construction of the tree, the left endpoint satisfies $|L-v^\star|\le 1/T$, so posting $L$ over the whole horizon contributes at most constant regret. It remains to control the number $N_R$ of times $R$ is posted. Between counter decreases, the probability of selecting $R$ decays as $\mathcal{O}(1/s_\ell)$. The sum of these probabilities is bounded by $\mathcal{O}(\log T\log(T/\delta))$, plus a constant for each
counter decrease. Since each decrease on $\ell^\star$ requires
a corrupted round, the additional contribution is
$\mathcal{O}(C)$. Finally, employing the Bernstein's inequality gives,
with high probability, the desired bound
$R(\ell^\star)\leq
\mathcal{O}\left(C+\log T\log(T/\delta)\right)$.

Then, we focus on bounding the regret associated with $\mathcal{L}^+$.

\begin{restatable}{lemma}{unknownplus}\label{lm:RCK_unknown_plus}
Algorithm \ref{alg:meta} with \textnormal{\textsc{Commit}} procedure \textnormal{\textsc{CommitUnknown}} guarantees:
\[
\sum_{\ell\in\mathcal L^{+}} R(\ell)\le 2N_{\textsc{F}}+2C.
\]
\end{restatable}
For the leaves above $\ell^\star$, an uncorrupted first round
immediately triggers a failure, since $L_\ell>v^\star$.
Hence, there are at most $N_{\textsc{F}}$ blocks with an
uncorrupted first round. Every remaining block has a corrupted
first round, so there are at most $C$ such blocks.
Since each block consists of at most two rounds, this gives
$\sum_{\ell\in\mathcal L^+}R(\ell)\le2N_{\textsc{F}}+2C$.

We finally bound the regret associated with the set $\mathcal{L}^-$.
\begin{restatable}{lemma}{unknownminus}\label{lm:RCK_unknown_minus}
Let $\delta\in(0,1)$. Algorithm \ref{alg:meta} with \textnormal{\textsc{Commit}} procedure \textnormal{\textsc{CommitUnknown}} guarantees:
\[
\sum_{\ell\in\mathcal L^{-}} R(\ell)\le64C+129N_{\textsc{F}}+65,
\]
with probability at least $1-\delta/3$.
\end{restatable}
Bounding the regret for the leaves in $\mathcal L^{-}$ is the most delicate case. Indeed, posting the left endpoint $L_\ell$ does not reveal that the leaf is incorrect. Since $L_\ell\le v^\star$, the learner observes a sale and may remain on the leaf. Posting the right endpoint $R_\ell$ in a non-corrupted block instead triggers a failure and decreases the counter. Therefore, if the learner spends many blocks on such a leaf while only few of them are corrupted or trigger a failure, the accumulated counter decreases must be unusually small. A martingale concentration argument bounds this event uniformly over all leaves and all possible values of $C_\ell+F_\ell$, where $F_\ell$ counts failed commitments on $\ell$. This shows that the number of completed blocks spent on each $\ell\in\mathcal L^{-}$ is, with high probability, of order $C_\ell+F_\ell+1$, and summing over leaves yields
$
\sum_{\ell\in\mathcal L^{-}}R(\ell)
\le\mathcal O\left(C+N_{\textsc{F}}+1\right).
$

Combining the previous results, we are ready to provide the final regret bound when $C$ is unknown.
\begin{restatable}{theorem}{unknown}\label{thm:unknown_regret}
Let $\delta\in(0,1)$. Assume that the feedback is corrupted in at most $C$ rounds. Algorithm~\ref{alg:meta} with \textnormal{\textsc{Commit}} procedure \textnormal{\textsc{CommitUnknown}} guarantees:
\[R_T\le  \mathcal{O}\left(C + \log T\log(T/\delta)\right),\]
with probability at least $1-\delta$.
\end{restatable}
Theorem~\ref{thm:unknown_regret} provides the main result of the paper, showing that Algorithm~\ref{alg:meta}, when instantiated with the procedure \textnormal{\textsc{CommitUnknown}}, attains a linear dependence on the corruption $C$ even when $C$ is \emph{not} known to the learner, thus effectively decoupling corruption and time.

\section{Discussion and Open Problems}

In this work, we show that, in robust pricing, it is possible to decouple the dependence on the corruption level $C$ from the time horizon $T$, resolving the open problem posed by~\citet{gupta2025robust}. By adopting a robustified binary-search approach augmented with tailored commitment phases, we obtain a regret bound of $\mathcal{O}(C+\log T)$ in the known-corruption setting. Furthermore, by introducing a novel randomized commitment strategy that carefully balances exploration and exploitation on candidate leaves, we achieve a regret of order $\mathcal{O}(C+\log^2 T)$ even when the corruption budget is unknown.
Our results leave two main directions open. First, it remains to determine whether the $\mathcal O(C+\log^2 T)$ bound in the unknown corruption setting is tight, either by proving a matching lower bound or by developing improved algorithms. Second, an important direction is to extend our techniques to the contextual setting considered by \cite{krishnamurthy2023contextual,gupta2025robust}.

\bibliographystyle{plainnat}
\bibliography{example_paper.bib}

\newpage


\appendix

\section{Additional Related Works}\label{sec:related_works}
Learning to price was initiated by the seminal work of~\citet{kleinberg2003value}, which provided matching upper and lower bounds for the regret achievable in the classical dynamic pricing problem. In particular, they showed that, in the uncorrupted one-dimensional setting, the optimal regret scales as $\Theta(\log\log T)$. This doubly-logarithmic dependence on the horizon is one of the most distinctive features of dynamic pricing, and reflects the fact that binary-search-type strategies are exceptionally well aligned with the threshold structure induced by the unknown buyer valuation.
Dynamic pricing have been then extended to stochastic or partially specified demand models~\citep{besbes2009dynamic, javanmard2019dynamic} and to richer settings such as pricing in repeated-auction and under strategic behavior~\citep{amin2013learning,cesabianchi2015regret, Drutsa17}. 

A further line of work studied higher-dimensional variants of the same learning problem under binary threshold feedback, often under the broader name of \emph{contextual pricing} or \emph{contextual search}. In these settings, the buyer's valuation depends on side information describing the item, but the learner still observes only one-bit sale/no-sale feedback and faces the same asymmetric pricing loss. A sequence of works progressively sharpened the achievable regret bounds in this richer model, including~\citet{cohen2020feature},~\citet{lobel2018multidimensional},~\citet{leme2022contextual}, and~\citet{liu2021optimal}, culminating in the $\mathcal{O}(d\log\log T)$ guarantee of~\citet{liu2021optimal}, where $d$ is the dimension of the context.

A closely related research line concerns learning under corrupted feedback. This line of work studies settings in which the underlying structure of the problem remains well defined, but a bounded number of observations may be unreliable or even chosen adversarially. Classical antecedents of this idea go back to Ulam's game~\citep{ulam1976adventures}.
More recently, corruption-robust guarantees have been investigated for higher-dimensional search and pricing problems under binary feedback. In particular,~\citet{krishnamurthy2023contextual} pioneered the study of corrupted contextual pricing through the lens of corrupted multidimensional binary search, obtaining guarantees of order $\mathcal{O}(d^3(C+\log T)\log T \log(d/\varepsilon))$ in that more general setting. 
Subsequently,~\citet{pmlr-v178-leme22a} introduced a density-based approach to corruption-robust contextual search, obtaining regret $\mathcal{O}(C+d\log(1/\varepsilon))$ for the $\varepsilon$-ball loss and $\mathcal{O}(C+d\log T)$ for a symmetric loss. Furthermore, the extended version of this work~\citep{leme2026density} clearly notes that extending density-based methods to the dynamic pricing setting appears to require significant novel approaches. The most closely related work is~\citet{gupta2025robust}, which sharpens guarantees for robust contextual pricing, obtaining regret $\mathcal{O}(Cd\log\log T)$ under known corruption, ruling out $\mathcal{O}(C+d\log\log T)$, and achieving $\mathcal{O}(C+\log T)$ in the one-dimensional cautious-buyer model with one-sided corruptions.


\section{Robust Binary Search}
\label{app:robust_binary_search}

In this section, we show that the robust binary-search technique of~\citep{rivest1980coping} can be used to attain
a $\mathcal O(C+\log T)$ regret bound for the known-corruption version
of our pricing problem. Specifically,~\citet{rivest1980coping} study the
problem of identifying an unknown element $x\in\{1,\ldots,n\}$ by using
comparison queries, when up to $C$ answers may be erroneous and $C$ is
known. Their theorem is usually stated in the sharper form:
\[
\log n + C\log\log n + \mathcal O(C\log C).
\]
For our purposes, however, the exact upper bound in their Theorem~2 is more
convenient. In particular, their result implies that the number of comparisons
needed in the worst case is at most:
\begin{align}\label{eq:rivest}
	Q(n,C)
	<
	\min\left\{
	Q' :
	2^{Q'-C}
	>
	n \sum_{i=0}^{C} \binom{Q'-C}{i}
	\right\}.
\end{align}
We now show that this expression is of order $\mathcal O(C+\log n)$.
\begin{lemma}
	\label{lem:rivest-coarse-bound}
	For every $n\geq 2$ and every $C\geq 0$, the robust binary-search procedure
	of~\citet{rivest1980coping} identifies an unknown element $x\in\{1,\ldots,n\}$ under at most $C$ erroneous comparison answers using $
	\mathcal O(\log n + C)
	$
	comparisons in the worst case.
\end{lemma}

\begin{proof}
	We use the exact upper bound from
	\citet{rivest1980coping} provided in Equation~\eqref{eq:rivest}. Thus, to prove the lemma, it is enough to exhibit a value $Q'=\mathcal{O}(C +\log n)$ such that:
	\[
	2^{Q'-C}
	>
	n \sum_{i=0}^{C}\binom{Q'-C}{i}.
	\]
	Choose
	\(Q' = C+\left\lceil 16\bigl(\lceil \log n\rceil+C\bigr)\right\rceil\).
	In the following, we show that this value satisfies the required condition.
	
	If $C=0$, then $\sum_{i=0}^{C}\binom{Q'-C}{i}=1$, and the condition reduces
	to $2^{Q'}>n$, which holds by the choice of $Q'$. Hence assume $C\geq 1$.
	
	By construction, it holds:
	\[
	\frac{C}{Q'-C}
	\leq
	\frac{C}{16(\lceil \log n\rceil+C)}
	\leq
	\frac1{16}.
	\]
	Let $H(p)=-p\log p-(1-p)\log(1-p)$ be the binary entropy function.
	Since $C/(Q'-C)\leq 1/16\leq 1/2$, the standard entropy bound for binomial
	sums gives:
	\[
	\sum_{i=0}^{C}\binom{Q'-C}{i}
	\leq
	2^{(Q'-C)H(C/(Q'-C))}.
	\]
	The function $H$ is increasing on $[0,1/2]$, so
	$H(C/(Q'-C))\leq H(1/16)<1/2$. Therefore,
	\[
	\sum_{i=0}^{C}\binom{Q'-C}{i}
	\leq
	2^{(Q'-C)/2}, \quad  n \sum_{i=0}^{C}\binom{Q'-C}{i}
	\leq
	2^{\lceil \log n\rceil} 2^{(Q'-C)/2}
	\]
	Moreover, again by the choice of \(Q'\),
	\[
	\frac{Q'-C}{2}
	=
	\frac{\left\lceil 16(\lceil \log n\rceil+C)\right\rceil}{2}
	>
	\lceil \log n\rceil .
	\]
	Hence $2^{\lceil \log n\rceil} 2^{(Q'-C)/2} < 2^{Q'-C}$. Combining the last two
	inequalities gives:
	\[
	n \sum_{i=0}^{C}\binom{Q'-C}{i}
	<
	2^{Q'-C}.
	\]
	Thus, the condition in the upper bound of \citet{rivest1980coping} is satisfied.
	Consequently,
	\[
	Q(n,C)
	\leq
	C+\left\lceil 16\bigl(\lceil \log n\rceil+C\bigr)\right\rceil
	=
	\mathcal O(\log n+C).
	\]
	This concludes the proof.
\end{proof}

We now translate this robust-search guarantee into a regret guarantee.

\begin{theorem}
	\label{thm:known-c-rivest-pricing}
	Consider the robust dynamic pricing problem with valuation
	$v^\star\in[0,1]$, horizon $T$, and at most $C$ corrupted feedback
	observations. If $C$ is known to the learner, then there exists a pricing
	algorithm with regret:
	\[
	R_T \leq \mathcal O(C+\log T).
	\]
\end{theorem}

\begin{proof}
	The result follows by partitioning $[0,1]$ into $T$ intervals of length $1/T$ and running the algorithm of~\citet{rivest1980coping} over this discretization. Once the optimal interval has been identified, the algorithm commits to the left endpoint of that interval, incurring regret at most $1/T$ on each subsequent round. Since during the search each round incurs regret at most $1$, the final bound is:
	\[
	R_T \leq Q(n,C) + \frac{1}{T}\cdot T \leq \mathcal{O}(C+\log(T)),
	\]
	by Lemma~\ref{lem:rivest-coarse-bound} with $n=T$. This concludes the proof.
\end{proof}

\section{Omitted Proofs of Section~\ref{sec:algorithm}}

\honeststep*
\begin{proof}
	We distinguish two cases depending on whether $I^\ind$ is (i) on-path or off-path:
	\begin{enumerate}[(i)]
		\item On-path. Since the step is honest, \textsc{SafetyCheck}$(I^\ind)$ returns \textsc{PASS} and the midpoint query is uncorrupted. Hence, the algorithm moves to the unique on-path child of $I^\ind$, which is one step closer to $\ell^\star$. Thus, $\Phi_{\ind+1}=\Phi_\ind-1$.
		\item Off-path. Being $I^\ind=[L,R)$ off-path,  either $R\le v^\star$ or $L>v^\star$. In either case, \textsc{SafetyCheck}$(I^\ind)$ returns \textsc{FAIL}, and the algorithm backtracks to the parent interval of $I^\ind$.
		Since $\ell^\star$ is not in the subtree rooted at $I^\ind$, the parent is one step closer to $\ell^\star$. Therefore, $\Phi_{\ind+1} = \Phi_\ind - 1$.
	\end{enumerate}
	This concludes the proof.
\end{proof}

\corruptedstep*
\begin{proof}
	In a single search step, the algorithm either moves to a child interval or backtracks to its parent interval. Consequently, the distance to $\ell^\star$ changes by  one in absolute value. Therefore, $\Phi_{\ind+1} \le \Phi_\ind + 1$.
\end{proof}

\leaf*
\begin{proof}
	A commit search step terminates when the commit procedure returns \textsc{FAIL} and the algorithm backtracks to the parent interval. If $\ell^\ind  \neq \ell^*$, the leaf is not the target, and its parent lies on the path to $\ell^*$, decreasing the distance from $\ell^\star$ and the potential by $1$.
	If $\ell^\ind = \ell^*$, the algorithm moves from the target interval to its parent (which has distance $1$ from $\ell^\star$). Hence, the potential is increased  by $1$.
\end{proof}

\meta*
\begin{proof}
	Let $H$ and $K$ denote the number of honest and corrupted search steps occurring at non-leaf intervals, respectively. Since each corrupted search step contains at least one corrupted round, we have $K \le C$. Let $q^\star$ be the final search step of the algorithm.
	We can express the final potential using a telescoping sum of the potential changes:
	\[ \Phi_{\ind^\star}= \Phi_0 + \sum_{\ind \in [\ind^\star]} \Phi_{\ind}-  \Phi_{\ind-1}.\]
	Applying \Cref{lem:honest-step}, \Cref{lem:corrupted-step}, and \Cref{lem:leaf} to bound the potential change in different types of search steps and observing that $\Phi_0=D$ we get:
	\[\Phi_{\ind^\star}=\Phi_0+ \sum_{\ind \in [\ind^\star]} \Phi_{\ind}-  \Phi_{\ind-1} \le D + K-H+ N_{\textsc{T}} - N_{\textsc{F}} \]
	Since the potential is by definition non negative, $\Phi_{\ind^\star}\ge0$, and hence:
	\[ N_{\textsc{F}} + H \le D + K + N_{\textsc{T}}.  \]
	By observing that $K\le C$ and $D=\log T$, we get: 
	\begin{align}\label{eq:boundH}
		N_{\textsc{F}} + H\le  \log T + C + N_{\textsc{T}}.
	\end{align}
	This immediately yields the first result: 
	\[  N_{\textsc{F}}\le D+K+N_{\textsc{T}}.\]
	To bound the total regret, note that search step at a non-leaf interval consists of at most three rounds (two for the safety check and one for the midpoint), with each round incurring a regret of at most $1$.
	We decompose the total regret into the regret from the regret on commitment steps and search step at non-leaf interval:
	\[ 
	R_T\leq \sum_{\ell \in \mathcal{L}} R(\ell) + 3 (H+ K)   \le  \sum_{\ell \in \mathcal{L}} R(\ell) + 3(\log T+C+N_{\textsc{T}}+C)=  \sum_{\ell \in \mathcal{L}} R(\ell) + 3\log (T )+6C+3N_{\textsc{T}},\]
	where in the second inequality we used \Cref{eq:boundH} and $K\le C$.
	This concludes the proof.
\end{proof}

\section{Omitted Proofs of Section~\ref{sec:known}}
\RCK*
\begin{proof}
	The regret of the commitment phase is the sum of regret on incorrect leaves and on the correct leaf $\ell^\star$.
	
	Consider a fixed incorrect leaf $\ell=[L,R) \in \mathcal{L} \setminus \{\ell^*\}$. At any iteration of the commit loop at leaf $\ell$ with $s_\ell\le C$, either the safety check fails and backtracks, or it passes and increments $s_\ell$. If it passes, then at least one of the checks on the two extremes is corrupted. Hence, it must be the case that at the end of time horizon holds $s_\ell \le C$. 
	This implies that on a commit search step on an incorrect leaf, it is always the case that $s_\ell\le C$. This implies that for a commit search step, all the iterations of the loop except (possibly) the last one are corrupted, namely, if the step last for $a$ iteration of the loop then the corruption is at least $a-1$. Since each iteration includes exactly two rounds we get:
	\[  \sum_{\ell \in \mathcal{L}\setminus \ell^\star} R(\ell) \le 2C+ 2 N_{\textsc{F}}.  \] 
	Consider now the correct leaf $\ell^\star=[L,R)$. As we noticed in Observation \ref{obs:comK}, the iterations of the commit procedure on $\ell^\star$ that fail and return are at most $C$. Hence, after at most $2C+1$ iteration, we get that $s_\ell>C$ and the algorithm commit to $L$.
	Since each iteration of the safety check requires 2 rounds, the regret before committing to $L$ is at most $4C+2$. Finally, once the algorithm commit to $L$ gets per-round regret at most $1/T$ since $v^\star-1/T \le L \le v^\star$.
	
	Overall, we get:
	\[  \sum_{\ell \in \mathcal{L}} R(\ell) \le 2C+2 N_{\textsc{F}} + 4C+2+ T\frac{1}{T}\le 2 N_{\textsc{F}} + 6C + 3.  \] 
	This concludes the proof.
\end{proof}

\known*

\begin{proof}
	
	Combining our results we get:
	\begin{align*}
		R_T&\leq \sum_{\ell \in \mathcal{L}} R(\ell) + 3\log T+6 C +3 N_{\textsc{T}}\\
		&\le 2 N_{\textsc{F}}+ 6C+3+ 3\log T+3 N_{\textsc{T}}+6 C\\
		& \le 2 (\log T+C+N_\textsc{T}) + 6C+3+ 3\log T+3 N_{\textsc{T}}+6 C\\
		& \le   5\log T+ 19 C + 3,
	\end{align*}
	where in the first inequality we use \Cref{thm:meta}, in the second inequality we use \Cref{lm:RCK}, in the third again \Cref{thm:meta}, and in the fourth Observation \ref{obs:comK}.  
	This concludes the proof.
\end{proof}

\section{Omitted Proofs of Section~\ref{sec:unknown}}

\unknownoptimal*
\begin{proof}
	To bound the regret associated to $\ell^\star=[L,R)$, we first notice that, by structure of the tree built by Algorithm~\ref{alg:meta}, it holds $|L-v^\star|\leq \frac{1}{T}$, thus the regret over the entire learning dynamic incurred when the left point is chosen is bounded by $1$.
	Thus, in order to bound $R(\ell^\star)$, we need to bound, in high probability, the number of times $R$ is selected.
	Now, let $N_R$ be the number of times the right prices have been posted. 
	
	To account for the counter decreases, we consider the quantity:
	\[
	4\ln(T/\delta)\ln\left(1+\frac{s_\ell}{4\ln(T/\delta)}\right).
	\]
	For $T\ge3$, a direct computation shows that each increment of the counter increases this quantity by at least $q_{s_\ell}/2$, where $s_\ell$ is the value after the increment, whereas each decrease reduces it by at most $11$. On the correct leaf, this can happen only if a round is corrupted, so there are at most $C$ decreases. Since the counter starts at zero and never exceeds $T$, we can bound the sum of the conditional probabilities of selecting $R$ as follows:
	\[
	\sum q_{s_\ell}
	\le8\ln(T/\delta)\ln\left(1+\frac{T}{4\ln(T/\delta)}\right)+22C
	\le8\ln T\ln(T/\delta)+22C=: \overline Q,
	\]
	where the sum is over the second rounds played on $\ell^\star$. 
	
	Now, we want to relate this bound to the true value. Since the draws are conditionally Bernoulli and the sum of their conditional selection probabilities is bounded by $\overline Q$, the conditional version of Bernstein's inequality gives that, for every $x>0$:
	\[
	\mathbb{P}\left(N_R\ge \overline Q+\sqrt{2\overline Q x}+\frac{x}{3}\right)\le e^{-x}.\]
	Applying this with $
	x:=\ln(3/\delta),$
	we obtain:
	\[
	\mathbb{P}\left(
	N_R\ge \overline Q+\sqrt{2\overline Q\ln(3/\delta)}+\frac13\ln(3/\delta)
	\right)\le \frac{\delta}{3}.
	\]
	Using $
	\sqrt{2\overline Q\ln(3/\delta)}
	\le
	\overline Q+\frac12\ln(3/\delta),
	$
	we get that, with probability at least $1-\delta/3$:
	\[
	N_R
	\le
	2\overline Q+\frac56\ln(3/\delta).
	\]
	Since $T\ge 3$ implies $\ln(3/\delta)\le \ln(T/\delta)$, it follows that, with probability at least $1-\delta/3$,
	$
	N_R
	\le
	44C+20\ln T\,\ln(T/\delta).
	$ This leads to a final regret bound of order:
	\[R(\ell^\star)\leq 1+44C+20\ln T\,\ln(T/\delta),\]
	which holds with probability at least $1-\delta/3$.
\end{proof}

\unknownplus*
\begin{proof}
	When $\ell\in\mathcal{L}^+$, the bound follows from the following observations. The regret on these leaves is bounded by the total number of rounds spent on them, and these rounds can be grouped into blocks of at most two, since the commit procedure may return after the first round.
	We now decompose these blocks into two types. The first type corresponds to blocks whose first round is uncorrupted. For a leaf $\ell\in\mathcal L^{+}$, the first round of every block posts $L_\ell$. Since $L_\ell>v^\star$, if that round is not corrupted then the learner observes a no-sale and immediately leaves the leaf. Hence the total number of blocks of the first type is at most $N_{\textsc{F}}$, where $N_{\textsc{F}}$ denotes the total number of failed commitments on non-optimal leaves.
	The second type corresponds to blocks whose first round is corrupted. It follows that the total number of blocks of the second type, summed over all leaves in $\mathcal L^{+}$, is at most $C$, including a possible final unfinished block.
	Since each block consists of at most two rounds which implies a regret at most $2$, we conclude:
	\[
	\sum_{\ell\in\mathcal L^{+}} R(\ell)\le 2N_{\textsc{F}}+2C.
	\]
	This concludes the proof.
\end{proof}

\unknownminus*

\begin{proof}
	Bounding the regret for the leaves in $\mathcal{L}^{-}$ requires a different approach compared to the previous cases. Unlike for Lemma~\ref{lm:RCK_unknown_plus}, we cannot directly relate the number of times $L$ is played to the corruption because, in the absence of corruption, posting $L$ does not trigger any backtracking. Furthermore, unlike for Lemma~\ref{lm:RCK_unknown_optimal}, we do not need to show that the probability of playing $R$ is small; instead, we must show that this probability is large enough to detect as soon as possible whether the leaf has experienced corruption.
	
	To do so, for every $\ell\in\mathcal L^{-}$ we denote by $N_\ell$ the total number of completed two-round search blocks spent on $\ell$, by $C_\ell$ the total number of corrupted search blocks experienced on $\ell$, and by $F_\ell$ the number of times the commitment procedure returns \textsc{FAIL} on $\ell$. Notice that these are random variables and, in general, they are not independent. We consider $T\ge3$, since the statement is immediate for $T\le2$. For every integer $h\ge0$, let
	$
	m_h:=32(h+1),
	$
	and define the event
	$
	\mathcal E_{\ell,h}:=\{N_\ell>m_h,\ C_\ell+F_\ell\le h\}.
	$
	
	We first show that $\mathbb P(\mathcal E_{\ell,h})$ is small. Fix $\ell\in\mathcal L^{-}$ and fix $h\ge0$. On the event $\mathcal E_{\ell,h}$, the learner plays the first $m_h$ completed blocks on $\ell$, possibly across different visits. Each non-corrupted query of $R_\ell$ triggers a failure and decreases the counter, whereas each completed block increases the counter by one before the second round.
	
	We now formalize this argument. Let $\mathcal F_n$ denote the history up to the moment in which the learner randomizes in the second round of the $n$-th completed block on $\ell$, including the adversary's corruption decision, and let $B_n$ be the Bernoulli random variable indicating whether the learner posts $R_\ell$ in that round. Since the counter can only decrease between consecutive increments, its value before this draw is at most $n$. Thus, conditional on $\mathcal F_n$, we have:
	\[
	\mathbb P(B_n=1\mid\mathcal F_n)=q_n,
	\qquad q_n\ge\min\left\{\frac{4\ln(T/\delta)}{n},1\right\}.
	\]
	Let $S$ denote the set of indices whose second round is corrupted. 
	
	We now sum the counter updates over the first $n$ completed blocks. Each non-corrupted query of $R_\ell$ decreases the counter by $4/q_n$, unless the update is clipped at zero. Since $4\ln(T/\delta)>4$, clipping can occur only when $q_n=1$, and reduces the decrement by at most $4$ per failure. The counter starts at zero, receives $n$ increments and stays nonnegative. Therefore, we obtain:
	\[
	\sum_{\substack{i\le n\\i\notin S}}\frac{B_i}{q_i}\le\frac n4+F_\ell.
	\]
	Since at most $C_\ell$ of these second rounds are corrupted,
	at least $n-C_\ell$ of them are uncorrupted. On $\mathcal E_{\ell,h}$, taking $n=m_h=32(h+1)$, we also have $C_\ell+F_\ell\le h<n/32$. Thus, the previous bound gives:
	\[
	\sum_{\substack{i\le n\\i\notin S}}\left(1-\frac{B_i}{q_i}\right)
	\ge\frac34n-C_\ell-F_\ell
	\ge\frac{23}{32}n.
	\]
	To bound $\mathbb P(\mathcal E_{\ell,h})$, we therefore bound
	the probability that the sum above is at least $23n/32$.
	
	Since $\mathbb E[B_i\mid\mathcal F_i]=q_i$, each term $1-B_i/q_i$ has conditional mean zero and is at most $1$. Its conditional variance satisfies:
	\[
	\operatorname{Var}\left(1-\frac{B_i}{q_i}\,\middle|\,\mathcal F_i\right)
	=\frac{1}{q_i}-1
	\le\max\left\{\frac{i}{4\ln(T/\delta)}-1,0\right\},
	\]
	where the inequality follows from the lower bound on $q_i$. In particular, the variance is zero when $i\le4\ln(T/\delta)$, because then $q_i=1$.
	
	To apply Freedman's inequality, we keep these increments only in the first $m_h$ second rounds on $\ell$, using zero on corrupted rounds and in all other rounds. Whether a round is corrupted or belongs to $\ell$ is determined before its Bernoulli draw. Thus, the increments still have conditional mean zero, even when the visits to $\ell$ depend on previous feedback.
	
	If $m_h\le4\ln(T/\delta)$, the sum is identically zero and cannot reach $23m_h/32$, so $\mathcal E_{\ell,h}$ is empty. Otherwise, summing the variance bounds gives:
	\[
	2\sum_{i=1}^{m_h}\max\left\{\frac{i}{4\ln(T/\delta)}-1,0\right\}
	+\frac{23m_h}{48}
	\le\frac{m_h^2}{4\ln(T/\delta)}.
	\]
	The left-hand side bounds the denominator in Freedman's inequality: twice the sum of the conditional variances, plus $2/3$ of the deviation $23m_h/32$. Applying the Freedman's inequality, we obtain:
	\[
	\begin{aligned}
		\mathbb P(\mathcal E_{\ell,h})\le\exp\left(-\frac{(23m_h/32)^2}{m_h^2/[4\ln(T/\delta)]}\right)\le\exp\bigl(-2\ln(T/\delta)\bigr)
		=\left(\frac\delta T\right)^2,
	\end{aligned}
	\]
	where we used $(23/32)^2\ge1/2$ in the last inequality.
	
	We can now take a union bound over all $\ell\in\mathcal L^{-}$ and all integers $h\ge0$. Since $N_\ell\le T/2$, only the values of $h$ satisfying $m_h\le T/2$ can give nonempty events. There are at most $T/64$ such values and at most $2T$ leaves. Thus, we obtain:
	\[
	\mathbb P\left(\bigcup_{\ell\in\mathcal L^{-}}\ \bigcup_{h\ge0}\mathcal E_{\ell,h}\right)
	\le\frac{T^2}{32}\left(\frac\delta T\right)^2
	=\frac{\delta^2}{32}
	\le\frac\delta3.
	\]
	Therefore, with probability at least $1-\delta/3$, none of the events $\mathcal E_{\ell,h}$ occurs. In what follows, we work on this good event.
	
	Fix now a visited leaf $\ell\in\mathcal L^{-}$. Since none of the events $\mathcal E_{\ell,h}$ occurs, we may choose $h=C_\ell+F_\ell$. If it were true that
	$
	N_\ell>m_{C_\ell+F_\ell},
	$
	then this would imply that $\mathcal E_{\ell,C_\ell+F_\ell}$ occurs, contradicting the good event. Therefore, we can conclude that
	$
	N_\ell\le m_{C_\ell+F_\ell},
	$
	which implies:
	\[
	N_\ell\le32(C_\ell+F_\ell+1),
	\]
	with probability at least $1-\delta/3$.
	
	We can now bound the regret associated with the set $\mathcal L^{-}$. Again, we remark that each completed search block consists of two rounds, and each round contributes regret at most $1$. Blocks that fail at the first round contribute at most $F_\ell$, and there can be at most one final unfinished block over the entire horizon. Since at most $N_{\textsc{F}}+1$ leaves in $\mathcal L^{-}$ are visited, summing over all leaves gives:
	\[
	\begin{aligned}
		\sum_{\ell\in\mathcal L^{-}}R(\ell)
		&\le2\sum_{\ell\in\mathcal L^{-}}N_\ell+N_{\textsc{F}}+1\\
		&\le64C+65N_{\textsc{F}}+64(N_{\textsc{F}}+1)+1\\
		&\le64C+129N_{\textsc{F}}+65.
	\end{aligned}
	\]
	This concludes the proof.
\end{proof}

\begin{lemma}\label{lm:RCK_unknown}
	Let $\delta\in(0,1)$. Algorithm \ref{alg:meta} with \textnormal{\textsc{Commit}} procedure \textnormal{\textsc{CommitUnknown}} guarantees:
	\[\sum_{\ell \in \mathcal{L}} R(\ell)\le 66+20\ln T\,\ln(T/\delta)+131N_{\textsc{F}}+110C,\]
	with probability at least $1-\delta$.
\end{lemma}

\begin{proof}
	The statement follows from combining Lemmas~\ref{lm:RCK_unknown_optimal},~\ref{lm:RCK_unknown_plus},~\ref{lm:RCK_unknown_minus}, with a final union bound.
\end{proof}

\unknown*
\begin{proof}
	Combining the previous results we get:
	\begin{align*}
		R_T&\leq \sum_{\ell \in \mathcal{L}} R(\ell) + 3\log T+6 C +3 N_{\textsc{T}}\\
		&\le 66+20\ln T\,\ln(T/\delta)+ 131N_{\textsc{F}}+ 110C + 3\log T+6 C +3 N_{\textsc{T}}\\
		&\le 66+20\ln T\,\ln(T/\delta)+ 131(\log T+ C+ N_{\textsc{T}})+ 110C + 3\log T+6 C +3 N_{\textsc{T}}\\
		&\le 66+20\ln T\,\ln(T/\delta)+ 131\log T+ 131C+ 131C+ 110C + 3\log T+6 C +3 C\\
		& \le 66+20\ln T\,\ln(T/\delta)+ 134\log T + 381C,
	\end{align*}
	where in the first inequality we use \Cref{thm:meta}, in the second inequality we use \Cref{lm:RCK_unknown}, which holds with probability at least $1-\delta$, in the third again \Cref{thm:meta}, and in the fourth Observation \ref{obs:unknown}. 
	This concludes the proof.
\end{proof}

\end{document}